\RequirePackage{etoolbox}
\newtoggle{conference}

\togglefalse{conference} 

\documentclass[10pt]{article}

\usepackage[paperwidth=7in, paperheight=10in, textwidth=5.25in, textheight=8.2in]{geometry}

\usepackage{subfigure}
\usepackage{amsmath}
\usepackage{amsthm}
\usepackage{amssymb}
\usepackage{natbib}
\usepackage{graphicx}
\usepackage{multirow}
\usepackage{booktabs}
\usepackage{mathtools}
\usepackage{xr,tikz}
\usepackage{enumitem}
\usepackage{float}
\usepackage{nicefrac}
\usepackage{forloop}
\usepackage[hidelinks]{hyperref}
\usepackage{url}
\usepackage{authblk}
\usepackage[capitalise]{cleveref}
\usepackage{environ}

\newtheorem{theorem}{Theorem}
\newtheorem{definition}{Definition}
\newtheorem{lemma}{Lemma}
\newtheorem{corollary}{Corollary}
\newtheorem{example}{Example}

\makeatother

\newcommand{\Algone}{{\textsc{Worcs-I}}\xspace}
\newcommand{\Algtwo}{{\textsc{Worcs-II}}\xspace}
\newcommand{\AlgtwoWeak}{{\textsc{Worcs-II-Weak}}\xspace}
\newcommand{\AlgtwoRank}{{\textsc{Worcs-II-Rank}}\xspace}
\newcommand{\AlgRand}{{\textsc{Random}}\xspace}
\newcommand{\AlgGBS}{{\textsc{Gts}}\xspace}
\newcommand{\AlgFastGBS}{{\textsc{Fast-Gts}}\xspace}
\newcommand{\AlgMinDist}{{\textsc{MinDist}}\xspace}
\newcommand{\Cstrong}{c_{\mathrm{strong}}}
\newcommand{\vv}{\mathcal{V}}
\newcommand{\m}{\mathcal{M}}
\DeclareMathOperator*{\supp}{supp}
\DeclareMathOperator*{\diam}{diam}
\DeclareMathOperator*{\argmin}{arg\,min}
\DeclareMathOperator*{\argmax}{arg\,max}
\usepackage{algorithm, algpseudocode, xspace, footnote}

\newcommand{\defcal}[1]{\expandafter\newcommand\csname c#1\endcsname{{\mathcal{#1}}}}
\newcommand{\defbb}[1]{\expandafter\newcommand\csname b#1\endcsname{{\mathbb{#1}}}}
\newcounter{calBbCounter}
\forLoop{1}{26}{calBbCounter}{
	\edef\letter{\Alph{calBbCounter}}
	\expandafter\defcal\letter
	\expandafter\defbb\letter
}

\title{Comparison Based Learning from Weak Oracles}
\usepackage{times}

\author[1]{Ehsan Kazemi}
\author[1]{Lin Chen}
\author[2]{Sanjoy Dasgupta}
\author[1]{Amin Karbasi}
\affil[1]{Yale Institute for Network Science\\Yale University}
\affil[2]{Jacobs School of Engineering, University of California San Diego}
\affil[ ]{\normalsize \texttt{\{ehsan.kazemi, lin.chen, amin.karbasi\}@yale.edu, 
		dasgupta@eng.ucsd.edu}}

\date{}
\begin{document}

\maketitle
\begin{abstract}
 There is increasing interest in learning algorithms that involve interaction between human and machine. Comparison-based queries are among the most natural ways to get feedback from humans. A challenge in designing comparison-based interactive learning algorithms is coping with noisy answers. The most common fix is to submit a query several times, but this is not applicable in many situations due to its prohibitive cost and due to the unrealistic assumption of independent noise in different repetitions of the same query.

In this paper, we introduce a new weak oracle model, where a non-malicious user responds to a pairwise comparison query only when she is quite sure about the answer. This model is able to mimic the behavior of a human in noise-prone regions.
We also consider the application of this weak oracle model to the problem of content search (a variant of the nearest neighbor search problem) through comparisons.
More specifically, we aim at devising efficient algorithms to locate a target object in a database 
equipped with a dissimilarity metric via invocation of the weak comparison oracle. 
We propose two algorithms termed \Algone and \Algtwo (Weak-Oracle Comparison-based Search), which provably locate the target object in a number of comparisons close to the entropy of the target distribution.
While \Algone provides better theoretical guarantees, \Algtwo is applicable to more technically challenging scenarios where the algorithm has limited access to the ranking dissimilarity between objects.
A series of experiments validate the performance of our proposed algorithms.
\end{abstract}

\section{Introduction}
Interactive machine learning refers to many important applications of machine learning that involve collaboration of human and machines. The goal of an interactive learning algorithm is to learn an unknown target
hypothesis from input provided by humans\footnote{In this paper, we refer to human, crowd or user interchangeably.} in the forms of labels, pairwise comparisons, rankings or numerical evaluations \citep{huang2010active,settles2010active,yan2011active}. 
A good algorithm,  in both theory and practice, should be able to efficiently deal with inconsistent or noisy data because human feedback can be erroneous. For instance, experimental studies show error rates up to 30\% in crowdsourcing platforms such as Amazon Mechanical Turk \citep{ipeirotis2010quality}. 

Studies conducted by psychologists and sociologists suggest that humans are bad at assigning meaningful numerical values to distances between objects or at rankings \citep{stewart2005absolute}. It has been shown that pairwise comparisons are a natural way to collect user feedback \citep{thurstone1927method,salganik2015wiki}: they tend to produce responses that are most consistent across users, and they are less sensitive to noise and manipulation \citep{maystre2015fast}. 
 Despite many successful applications of pairwise comparison queries in noise-free environments, little is known on how to design and analyze noise-tolerant algorithms.
Furthermore, many active learning methods end up querying points that are the most noise-prone \citep{balcan2009agnostic}.

In the literature, the most common approach to cope with noisy answers from comparison based oracles is to make a query several times and use majority voting \citep{dalvi2013aggregating}.
These methods assume that by repeating a question they can reduce the error probability arbitrarily.
As an example,  this approach has been considered in the context of classic binary search through different hypotheses and shown to be suboptimal \citep{karp2007noisy,nowak2009noisy}.
In general, the idea of repeated queries for handling noisy oracles suffers from three main disadvantages: (i) In many important applications, such as recommender systems and exploratory searches, the algorithm is interacting with only one human and it is not possible to incorporate feedback from several users. (ii) Many queries are inherently difficult, even for human experts,  to answer. (iii) Massive redundancy significantly increases the total number of oracle accesses in the learning process. Note that each query information requires significant cost and effort.

In order to efficiently address the problem of non-perfect and noisy answers from the crowd, we introduce a new \emph{weak oracle model}, where a non-malicious user responds to a pairwise comparison query only when she is quite sure about the answer
 More specifically, we focus on the application of this weak oracle model to queries of the form of \emph{``is object $z$ more similar to object $x$ or to object $y$?"}.   In this model, a weak oracle gives an answer only if one of the two objects $x$ or $y$ is substantially more similar to $z$ than the other. We make this important assumption to cope with difficult and error-prone situations where,  to almost the same degree, the two objects are similar to the target.
 Since the oracle may decline to answer a query, we can refer to it as an  \emph{abstention oracle}.
This model is one of the very first attempts to somewhat accurately mimic the behavior of crowd in real world problems. 
To motivate the main model and algorithms of this paper consider the following example.
\begin{example}
Assume Alice is a new user to our movie recommender system platform.
  In order to recommend potentially interesting movies to Alice, we plan to figure out an unknown target movie\footnote{The target movie of Alice is unknown to the platform.} that she likes the most.
  Our aim is to find her favorite movie by asking her preference over pairs of movies she has watched already. With each response from Alice, we expect to get closer to the target movie.
 For this reason, we are interested in algorithms that find Alice's taste by asking a few questions.
 To achieve this goal, we should (i) present Alice movies that are different enough from each other to make decision easier for her, and (ii) find her interest with the minimum number of pairwise comparisons because each query takes some time and the number of potential queries might be limited. 
\end{example}

To elaborate more on the usefulness of our weak oracle model, as a special case of active learning, we aim at devising an algorithm to locate a target object in a database through comparisons.
We model objects in the database as a set of points in a metric space.
In this problem, a user (modeled by a weak oracle) has an object $t$ (e.g., her favorite movie or music) in mind which is not revealed to the algorithm.
At each time, the user is presented with two candidate objects $x$ and $y$. Then through the query $\mathcal{O}_t(x, y)$ she indicates which of these is closer to the target object.
Previous works assume a \emph{strong comparison oracle} that always gives the correct answer \citep{tschopp2011randomized, karbasi2012comparison, karbasi2012hot}.
However, as we discussed earlier, it is challenging, in practice, for the user to pinpoint the closer object when the two are similar to the target to almost the same degree. 

We therefore propose the problem of comparison-based search via a \emph{weak oracle} that gives an answer only if one of the two objects is substantially more similar to the target than the other. 
Note that due to the weak oracle assumption, the user cannot choose the closer object if they are in almost the same distance to the target.
Our goal is to find the target by making as few queries as possible to the weak comparison oracle. 
We also consider the case in which the demand over the target objects is heterogeneous, i.e., there are different demands for object in the database.

Analysis of our algorithms relies on a measure of intrinsic dimensionality for the distribution of demand. These notions of dimension imply, roughly, that the number of points in a ball in the metric space increases by only a constant factor when the radius of the ball is doubled. This is known to hold, for instance, under limited curvature and random selection of points
from a manifold (two reasonable assumptions in machine learning applications) \citep{karger2002finding}.
 Computing the intrinsic dimension is difficult in general, but we do not need to know its value; it merely appears in the query bounds for our algorithms.
\paragraph{Contributions} In this paper, we make the following contributions: 
\begin{itemize}
	\item We introduce the weak oracle model to handle human behavior, when answering comparison queries, in a more realistic way. We believe that this weak oracle model opens new avenues for designing algorithms, which handle noisy situations efficiently, for many machine learning applications. 
		\item As an important example of our weak oracle model, we propose a new adaptive algorithm, called \Algone, to find an object in a database. This algorithm relies on a priori knowledge of the distance between objects, i.e., it assumes we can compute the distance between any two objects. 
		\item We prove that, under the weak oracle model, if the search space satisfies certain mild conditions then our algorithm will find the target in a near-optimal number of queries. 
		\item In many realistic and technically challenging scenarios, it is difficult to compute all the pairwise distances and we might have access only to a noisy information about relative distances of points.
	To address this challenge, we propose \Algtwo.  
	The main assumption in \Algtwo is that for all triplets $x,y$ and $z$ we only know if the the distance of one of $x$ or $y$ is closer to $z$ by a factor of $\alpha$. 
	Note this is a much less required information than knowing all the pairwise distances.
	This kind of information could be obtained, for example, from an approximation of the distance function.\footnote{Finding the partial relative distances between points is not the focus of this paper.} 
	We prove that \Algtwo, by using only this minimal information, can locate the target in a reasonable number of queries.
	We also prove that if \Algtwo is allowed to learn about the ranking relationships between objects, it provides better theoretical guarantees.
	\item	Finally, we evaluate the performance of our proposed algorithms and several other baseline algorithms over real datasets.
\end{itemize}

The remainder of this paper is organized as follows. We discuss the related work in \cref{sec:related}. We overview definitions and preliminaries in \cref{sec:def}. We present the main algorithms of this paper in \cref{sec:main}.
In \cref{sec:experiments}, we report our experimental evaluations. We conclude the paper in \cref{sec:conclusion}. 
\section{Related Work}\label{sec:related}
Interactive learning from pairwise comparisons has a range of applications in machine learning, including exploratory search for objects in a database \citep{tschopp2011randomized, karbasi2011content, karbasi2012comparison, karbasi2012hot}, nearest neighbor search \citep{goyal2008disorder}, ranking or finding the top-$k$ ranked items \citep{chen2013pairwise,wauthier2013efficient, eriksson2013learning, maystre2015fast, heckel2016active,maystre2017just}, crowdsourcing entity resolution  \citep{wang2012crowder,wang2013leveraging,firmani2016online}, learning user preferences in a recommender system \citep{wang2014active,qian2015learning}, estimating an underlying value function \citep{balcan2016learning}, multiclass classification \citep{hastie1997classification}, embedding visualization \citep{tamuz2015adaptively} and  image retrieval \citep{wah2014similarity}.

Comparison-based search, as a form of interactive learning, was first introduced by \citet{goyal2008disorder}.
 \citet{lifshits2009combinatorial} and \citet{tschopp2011randomized} explored this problem in worst-case scenarios, where they assumed demands for the objects are homogeneous.\footnote{Their distributions are close to uniform.}
\citet{karbasi2011content} took a Bayesian approach to the problem for heterogeneous demands, i.e, where the distribution of demands are far from the uniform.
It is noteworthy that the two problems of content search and finding nearest neighbor are closely related \citep{clarkson2006nearest,indyk1998approximate}. 

Search though comparisons can be formalized as an special case of the Generalized Binary Search (GBS) or splitting algorithm \citep{dasgupta2004analysis,dasgupta2005coarse}, where the goal is to learn an underlying function from pairwise comparisons. 
In practice GBS performs very well in terms of required number of queries, but it 
 appears to be difficult to provide tight bounds for its performance. Furthermore, its computational complexity is prohibitive for large databases. For this reason, \citet{karbasi2012comparison,karbasi2012hot} introduced faster algorithms that locate the target in a database with near-optimal expected query complexities. In their work, they assume a strong oracle model and a small intrinsic dimension for points in the metric space. 

The search process should be robust against erroneous answers.  \citet{karp2007noisy} and \citet{nowak2009noisy} explored the idea of repeated queries for the GBS to cope with noisy oracles. Also, for the special case of search though comparisons, \citet{karbasi2012comparison}  propose similar ideas for handling noise.

The notion of assuming a small intrinsic dimension for points in a metric space is used in many different problems such as  finding the nearest neighbor \citep{clarkson2006nearest,har2013approximate}, routing \citep{hildrum2004distributed,tschopp2015hierarchical}, and  proximity search \citep{krauthgamer2004navigating}. 

\section{Definitions and Preliminaries} \label{sec:def}
Let $ \mathcal{M} $ be a 
metric space with distance function $ d(\cdot,\cdot) $.
Intuitively, $ \mathcal{M}$ characterizes the database of objects. 
An example of a distance function is the Euclidean distance between feature vectors of items in a database.
Given any object $ x\in \mathcal{M} $, we define the ball centered at $ x $ with radius $ r $ as
$ B_x(r) \triangleq \{ y\in \mathcal{M} : d(x,y)\leq r \}.$
We define the diameter of a set $ A\subseteq \mathcal{M} $ as
$ \diam(A) = \sup_{x,y\in A} d(x,y). $
We will consider a distribution $ \mu $ over the set $ \mathcal{M} $ which we call the \emph{demand}: it is a non-negative function $ \mathcal{M} \to [0,1] $ with $ \sum_{x\in \mathcal{M} } \mu(x) $ = 1. 
For any set $ A\subseteq \mathcal{M} $, we define
$\mu(A) = \sum_{x\in A} \mu(x)$.
The \emph{entropy} of $ \mu $ is given by
$ H(\mu) = \sum_{x\in \supp(\mathcal{M})} \mu(x) \log \tfrac{1}{\mu(x)},$
 where $ \supp(\mu) \triangleq \{ x\in \mathcal{M} : \mu(x) >0 \} $ is the support of $ \mu $. 
  Next, we define an $\epsilon$-cover and an $\epsilon$-net for a set $\mathcal{A} \subseteq \mathcal{M}$.
  
  \begin{definition} [$\epsilon$-cover]
  	An $\epsilon$-cover of a subset $\mathcal{A} \subseteq \mathcal{M}$ is a  collection of points $\{x_1, \cdots , x_k\}$ of $\mathcal{A}$ such that $   \mathcal{A} \subseteq  \bigcup_{i=1}^k B_{x_i}(\epsilon)  $.
  \end{definition}
  
  \begin{definition} [$\epsilon$-net]
  	An $\epsilon$-net of a subset $\mathcal{A} \subseteq \mathcal{M}$ is a maximal (a set to which no more objects can be added) collection of points $\{x_1, \cdots , x_k\}$ of $\mathcal{A}$ such that for any $i \neq j$ we have $d(x_i, x_j ) > \epsilon$.
  \end{definition}
An $ \epsilon $-net can be constructed efficiently via a greedy algorithm in $ O(k^2 |\mathcal{A}|) $ time, where $ k $ is the size of the $ \epsilon $-net. We refer interested readers to~\citep{clarkson2006nearest}.
   Our algorithms are analyzed under natural assumptions on the expansion rate and underlying dimension of datasets, where we define them next.

 \paragraph{Doubling Measures}
 The \emph{doubling constant} captures the embedding topology of the set of objects in a metric space  when we are generalizing analysis from Euclidean spaces to more general metric spaces.
 Such restrictions imply that the volume of a closed ball in the space does not increase drastically when its radius is increased by a certain factor.
 Formally, the doubling constant of the measure $\mu$ \citep{karger2002finding}  is given by
 \[ c \triangleq \sup_{x\in \supp(\mu) ,  R\geq 0 } \frac{\mu(B_x(2R))}{\mu(B_x(R))}. \]
  A measure $ \mu $ is \emph{$c$-doubling} if its doubling constant is $ c < \infty $. We define $\textrm{dim}_m(\mathcal{M}) = \log c$.
To provide guarantees for the performance of algorithms in \cref{sec:partial-knowledge,sec:worcs-weak}, a stronger notion of doubling measure is required. Intuitively, we need bounds on the expansion rate of all the possible subsets of support of $\mu$.
The \emph{strong doubling constant} of the measure $ \mu $~\citep{ram2013which,haghiri2017comparison} is given by
\[ 
\Cstrong \triangleq \sup_{A \subseteq \mathcal{M}} \sup_{x\in \supp(\mu) \cap A , R\geq 0 } \frac{\mu(B_x(2R) \cap A)}{\mu(B_x(R) \cap A)}. \]
 We also say that measure $ \mu $ is \emph{$\Cstrong$-doubling} if its strong doubling constant is $ \Cstrong < \infty$. 
 Note that for a finite space $\mathcal{M}$ both $c$ and $\Cstrong$ are bounded.

 \paragraph{Doubling Spaces}
The doubling dimension of a metric space $(\mathcal{M}, d)$ is the smallest  $k$ such that every subset $B_x(2R)$ of $ \mathcal{M} $ can be covered by at most $k$ balls of radius $R$. We define the doubling dimension $\mathrm{dim}_s(\mathcal{M}) = \log_2 k$.
\cref{theorem:bounded_doubling_dimension} states that a metric with a bounded doubling constant has a bounded doubling dimension, i.e., every doubling measure carries a doubling space.

\begin{lemma}[Proposition 1.2.~\citep{gupta2003bounded}] \label{theorem:bounded_doubling_dimension}
	For any finite metric $(\mathcal{M}, d)$, we have 
	\[\mathrm{dim}_s(\mathcal{M}) \leq 4 \mathrm{dim}_m(\mathcal{M}),\]
	 where $ \mathrm{dim}_m(\mathcal{M}) \triangleq \log_2 c $.
\end{lemma}

\begin{proof}
	To prove this lemma, we show that any ball $B_x(2R)$ could be covered by at most $k \leq c^4$ balls of radius $R$.
	\begin{lemma} \label{theorem:eps-net-cover}
		An $ \epsilon $-net of  $\mathcal{A} $ is an $ \epsilon $-cover of $ \mathcal{A} $.
	\end{lemma}
	\begin{proof}
		Let $ \{x_1, x_2,\cdots, x_k\} $ be an $ \epsilon $-net of $ \mathcal{A} $.
		Suppose that there exists $  x\in \mathcal{A} $ such that $ x\notin \bigcup_{i=1}^k B_{x_i}(\epsilon) $. Then we know that $ d(x, x_i) > \epsilon $ for all $ i $. This contradicts the maximality of the $ \epsilon $-net.
	\end{proof}
	
	Assume $\{x_1, x_2,\cdots, x_k\}$ is an $R$-net of the ball $B_x(2R)$. We know for all $i \neq j$, $B_{x_i}(R/2) \cap B_{x_j}(R/2) = \varnothing$. Therefore, we have $\mu(\cup_{i = 1}^k B_{x_i}(R/2)) = \sum_{i = 1}^{k} \mu(B_{x_i}(R/2)).$ Also,
	\[ \mu(B_{x_i}(R/2)) \geq c^{-3} \mu(B_{x_i}{(4R)}) \geq c^{-3} \mu(B_x(2R)).  \]
	Also, $\cup_{i = 1}^k B_{x_i}(R/2) \subseteq B_x(4R)$. To sum-up we have
	\begin{align*}
	\mu(B_x(2R))\geq c^{-1} \mu(B_{x}(4R)) & \geq c^{-1} \mu(\cup_{i = 1}^k B_{x_i}(R/2)) \\
	& \geq c^{-1}  \sum_{i = 1}^{k} \mu(B_{x_i}(R/2)) \geq k c^{-4}  \mu(B_x(2R)).
	\end{align*}
\end{proof}

 \paragraph{Weak Comparison Oracle}
 The  \emph{strong comparison oracle} $\mathcal{O}^{\mathrm{strong}}_z(x,y)$ \citep{goyal2008disorder} takes three objects $ x,y,z $ as input
 and returns which one of $ x $ and $ y $ is closer to $ z $;
 formally, 
 \[ 
 \mathcal{O}^{\mathrm{strong}}_z(x,y)  =
 \begin{cases}
 x , & \mathrm{if}  \ d(x,z)\leq d(y,z), \\
 y, & \mathrm{otherwise}.
 \end{cases}
  \]
  
  In this paper, we propose  the \emph{weak comparison oracle}\footnote{We refer to it as an \emph{abstention oracle} interchangeably.} as a more realistic and practical model. This oracle gives an answer for sure only if one of $ d(x,z) $ and $ d(y,z) $ is substantially smaller than the other.
Formally,
  \[ \mathcal{O}_z(x , y )=
  \begin{cases}
  x, & \mathrm{if}  \ \alpha d(x,z) \leq d(y,z), \\
  y, & \mathrm{if} \ \alpha d(y,z) \leq d(x,z), \\
  ? \ \mathrm{or} \ x, &  \mathrm{if} \ d(x,z) \leq d(y,z) < \alpha d(x, z), \\
  ? \ \mathrm{or} \ y, &  \mathrm{if} \ d(y,z) < d(x,z) <  \alpha d(y, z).
  \end{cases} \]
 where $ \alpha \geq 1 $ is a constant that characterizes the approximate equidistance of $ x $ and $ y $.
 More precisely,  the answers to a query $\mathcal{O}_z(x ,y )$ has  the following interpretation:
 (i) if it is $x$, then $d(x, z) \leq d(y,z)$. (ii) if it is $y$,  $d(x, z) > d(y,z)$. (iii) if it is $?$,  $d(x, z)$ and $d(y, z)$ are within a multiplicative factor $\alpha$ of each other. 
 In our model, an answer of $x$ or $y$ does not provide any information about the relative distances of $x$ and $y$ from the target other than stating that one of them is closer to $t$. 
 Given $ x,y\in A \subseteq \mathcal{M} $, the $ \alpha $-weighted Voronoi cell of $ x $ is defined as
 $Vor(x, y, A) = \{v\in A : \alpha d(x,v) \leq d(y, v) \}, $
 where $ \alpha \geq 1 $.
 Similarly, we can define the $ \alpha $-weighted Voronoi cell of $ y $ as $ Vor(y, x, A) $.
 Note that the two Voronoi cells $Vor(x, y, A)$ and $Vor(y, x, A)$ contain the set of points which we are sure about the answers we get from the oracle. 
 \cref{fig:voronoi} shows answers we get from a weak oracle to the query $\mathcal{O}_z(x,y)$ for points $z$ in different regions of a 2-dimensional Euclidean space.
\begin{figure}[htb!]
	\centering
	\includegraphics[width =0.7\textwidth] {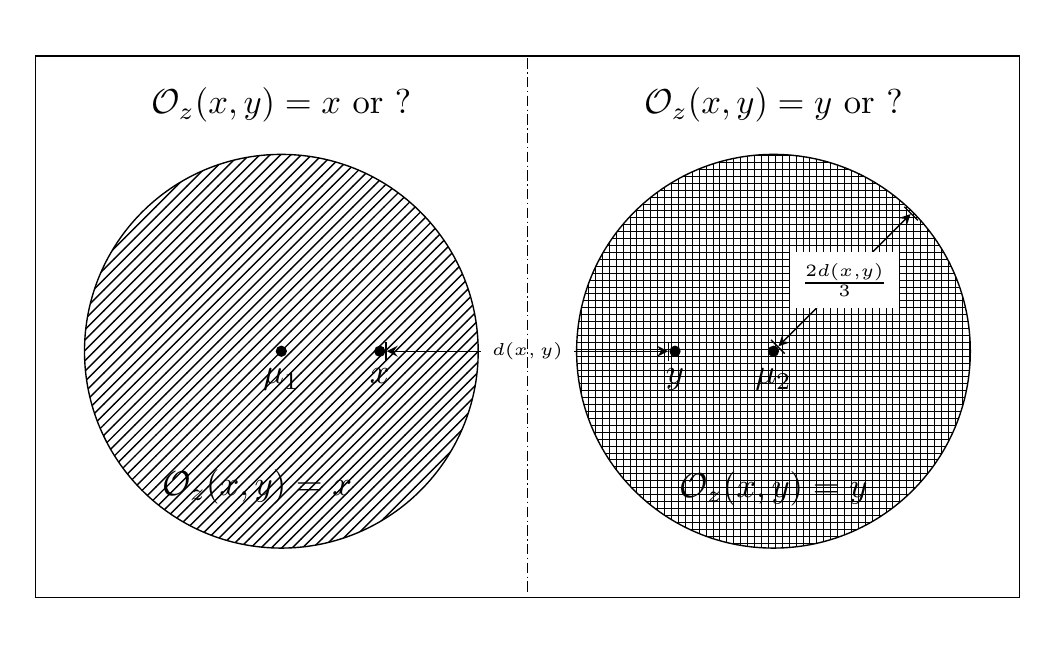}
	\caption{Voronoi cells with a multiplicative distance $\alpha = 2$ for a two dimensional Euclidean space. This figure partition the space into three parts based on the potential answers from the oracle.}\label{fig:voronoi}
\end{figure}

\section{Main Results} \label{sec:main}

Our main goal is to design an algorithm that can locate an unknown target $t \in \m$ by making only $f(c)H(\mu)$ queries in expectation, where $f(c)$ is polynomial in the doubling constant $c$.
In addition, we are interested in algorithms with a low computational complexity for deciding on the next query which we are going to make. 
Note that identifying a target having access to the very powerful \emph{membership oracle}\footnote{An oracle such that for any $A \subseteq \m$ identifies if $t \in A$ or not.} in average needs at least $H(\mu)$ queries \citep{cover2012elements}. 
Also, \citet{karbasi2011content} proved that a strong oracle needs $\Omega(cH(\mu))$ queries to find an object via a strong oracle. To sum-up, we conclude that an algorithm with a query complexity $f(c)H(\mu)$ is near-optimal (in terms of number of queries) for a weak oracle model.

\subsection{Generalized Ternary Search}\label{sec:gts}
In this section, we first take a greedy approach to the content search problem as a baseline.
Let the version space $\mathcal{V}_i$ be the set of points (hypotheses) consistent with the answers received from the oracle so far (at time $i$). 
After each query $\mathcal{O}_t(x,y)$ there are three possible outcomes of $``x"$, $``y"$ or $``?"$. 
The Generalized Ternary Search algorithm (\AlgGBS), as a generalization of GBS  \citep{dasgupta2004analysis}, greedily selects the query $\mathcal{O}_t(x,y)$ that results in the largest reduction in mass of the version space based on the potential outcomes. 
More rigorously, the next query to be asked by \AlgGBS is found through
\begin{align}\label{eq:gts}
\argmin_{(x,y) \in \mathcal{V}_i^2} \max( \sum_{z \in \vv_i, \mathcal{O}_z(x,y) = x} \mu(z),
  \sum_{z \in \mathcal{V}_i, \mathcal{O}_z(x,y) = y} \mu(z) , \sum_{z \in \mathcal{V}_i, \mathcal{O}_z(x,y) = ?} \mu(z)).
\end{align}
The computational complexity of GTS is $\Theta(n^3)$ and thus makes it prohibitive for large databases. 
Also, while this type of greedy algorithms performs well in practice (in terms of number of queries), it is difficult (or even impossible) to provide tight bounds for their performances \citep{dasgupta2004analysis,golovin2011adaptive}. 
Furthermore, the noisy oracle model makes analysis even harder.
For this reason, we need more efficient algorithms with provable guarantees.
The key ingredient of our methods is how to choose queries in order to find the target in a near optimal number of queries. 

\subsection{Full Knowledge of the Metric Space }

The first algorithm we present (\Algone) relies on a priori knowledge of the distribution $\mu$ and the pairwise distances between all the points. 
We prove that \Algone (see \cref{alg:full_knowledge}) locates the target $t$, where its query complexity is within a constant factor from the optimum.
\Algone uses two main ideas: (i) it guarantees that, by carefully covering the version space, in each iteration the target $t$ remains in the ball (which is found in line~\ref{line:next_version_space} of \cref{alg:full_knowledge}), and (ii) the volume of version space decreases by a constant factor at each iteration. These two facts ensure that in a few iterations we can locate the target.
We should mention that while \Algone relies on knowing all the pairwise distances, but we guarantee bounds for a $c$-doubling measure $\mu$ in the metric space $\mathcal{M}$.\footnote{This is a fairly loose condition for a metric space.}
Later in this section, we devise two other algorithms that relax the assumption of having access to the pairwise distances.

\begin{theorem}\label{thm:workcs-I}
	The average number of queries required by \Algone to locate an object is bounded from above by 
	\[ f(c) \cdot \left[1 + \tfrac{H(\mu)}{\log(1/ (1-c^{-2}) )} \right],\]
	where $c$ is the doubling constant of the measure $\mu$ in the metric space $\mathcal{M}$ and $f(c)$ is a polynomial function in $c$.
\end{theorem}

	\begin{proof}
	To prove this theorem, we show: (a) In each iteration the target $t$ remains in the next version space. Therefore, \Algone always find the target $t$.
	(b) The $\mu$-mass of the version space shrinks by a factor of at least $1 - c^{- 2}$ in each iteration, which results in the total number of $1 + \tfrac{H(\mu)}{\log(1/ (1-c^{-2}) )}$ iterations.
	(c) The number of queries to find the next version space is upper bounded by a polynomial function of the doubling constant $c$.
	
	(a) Assume $c_t$ is the center of a ball which contains $t$, i.e., $t \in B_{c_t}(\tfrac{\Delta}{8(\alpha + 1)})$.
	For all $c_j \neq c_t$ such that $d(c_t, c_j) > \tfrac{\Delta}{8}$, we have  $\mathcal{O}_t(c_t, c_j) = c_t$.
	This is because $d(c_t, t) \leq \tfrac{\Delta}{8(\alpha + 1)}$ and $d(c_j, t) > \tfrac{\Delta}{8} - \tfrac{\Delta}{8(\alpha + 1)}$, and thus $\alpha d(c_t, t) < d(c_j, t)$. 
	On the other hand, consider an $i$ such that for all $j \neq i$ with  $d(c_i, c_j) > \tfrac{\Delta}{8}$, we have $\mathcal{O}_t(c_i, c_j) = c_i$. We claim $t \in  B_{c_i}(\tfrac{\Delta(\alpha + 2)}{8(\alpha + 1)}).$
	Assume this in not true. Then for $c_t$ we have $d(c_i, c_t) > d(c_i,t) -  \tfrac{\Delta}{8(\alpha + 1)} \geq \frac{\Delta}{8}.$ Therefore, following the same lines of reasoning as the fist part of the proof, we  should have $\mathcal{O}_t(c_i, c_t) = c_t$. This contradicts our assumption. 
	
	(b) We have $\mu(\mathcal{V}_{i+1}) \leq (1 - c^{- 2}) \mu(\mathcal{V}_{i})$. 
		To prove this, we first state the following lemma.
		\begin{lemma}\label{lemma:maximumdist}
		Assume $\Delta = \max_{x,y \in \mathcal{V}_i} d(x,y)$. $ \forall x \in \mathcal{V}_i$ we have $\max_{y \in \mathcal{V}_i} d(x,y) \geq \tfrac{\Delta}{2}.$
	\end{lemma}
	\begin{proof}
		Assume $x^*, y^* = \argmax_{x,y \in \mathcal{V}_i} d(x,y)$.
		By using triangle inequality  we have $d(x, x^*) + d(x, y^*) \geq d(x^*, y^* ) =  \Delta$. We conclude that at least one of $d(x, x^*)$ or  $d(x, y^*) $ is larger than or equal to $\tfrac{\Delta}{2}$. This results in $ \max_{y \in \mathcal{V}_i} d(x,y) \geq \max \{ d(x, x^*) ,d(x, y^*) \} \geq \tfrac{\Delta}{2}$.
	\end{proof}
Let's assume $c_i$ is the center of $\mu(\mathcal{V}_{i+1})$ and  point $c_i^*$ is the furthest point from $c_i$. From \cref{lemma:maximumdist} we know that $d(c_i, c_i^*) \geq \frac{\Delta}{2}$.
	Also, it is straightforward to see $B_{c_i}(\frac{\Delta(\alpha + 2)}{8(\alpha + 1)}) \cap B_{c_i^*}(\frac{\Delta}{4}) = \emptyset$. From the definition of expansion rate we have $\mu(B_{c_i^*}(\frac{\Delta}{4})) \geq c^{-2} \mu(B_{c_i^*}(\Delta))  \geq c^{-2} \mu(\mathcal{V}_i)$.

	(c) Let's consider $t \in \supp \m $ as the target. \Algone  locates the target $t$, provided $\mu(\vv_i) \leq (1-c^{-2})^i \mu(\mathcal{V}_0) \leq \mu(t)$ or equivalently $ i \geq 1 + \tfrac{\log \mu(t) }{ \log(1-c^{-2} ) } \geq \lceil \tfrac{\log(\mu(t) }{ \log(1-c^{-2}) } \rceil $.
	The expected number of iterations is then upper bounded by $ \sum_{t \in \supp(\m)} \mu(t) \left( 1 + \tfrac{\log \mu(t) }{ \log(1-c^{-2} )}  \right) = 1 + \tfrac{H(\mu)}{\log(1/ (1-c^{-2}) )}.$ Finally, from \cref{theorem:bounded_doubling_dimension}, we know that we can cover the version space $\vv_i$ with at most $c^{4 \lceil \log 8(\alpha + 1) \rceil }$ balls of radius $\tfrac{\Delta}{8(\alpha + 1)}$. 
	Note that in the worst case we should query the center of each ball versus centers of all the other balls in each iteration.
\end{proof}

	In order to find a $ \frac{\Delta}{8(\alpha + 1)} $-cover in line~\ref{line:cover} of~\cref{alg:full_knowledge},  we can generate a $ (\Delta/2^l) $-net of $ \mathcal{V}_i $, where $ l =\lceil \log_2 8(\alpha+1) \rceil $. 
	The size of such a net is at most $ c^{\lceil \log_2 8(\alpha+1) \rceil + 3} $ by Lemma~1 in~\citep{karbasi2012hot}. Also, since we assume the value of $c$ is constant, $f(c)$ is independent of the size of database; it remains constant and our bound depends on the size of database only through $H(\mu)$, which is growing with a factor of at most $O(\log |M|).$

	\begin{algorithm}[htb!]
		\caption{\Algone \label{alg:full_knowledge}}
		\begin{algorithmic}[1]
			\Require Metric space $\mathcal{M}$ and oracle $\mathcal{O}_t(\cdot, \cdot)$
			\Ensure Target $t$
			\State $  \mathcal{V}_0 \gets \mathcal{M} $
			\State $ i\gets 0 $
			\Loop
			\If{$ |\vv_i| \leq 2 $}
			\State Invoke $ \mathcal{O}_t(x,y) $ and halt
			\EndIf
			\State $\Delta \gets $ diameter of $\mathcal{V}_i$ 
			\State Let $C$ be a $\frac{\Delta}{8(\alpha + 1)}$-cover of $\mathcal{V}_i$ \label{line:cover}
			\For{each $c_i \in C$}
			\If{for all $c_j$ such that $d(c_i, c_j) > \frac{\Delta}{8}$: $\mathcal{O}_t(c_i, c_j) = c_i$ }
			\State $\mathcal{V}_{i+1} \gets B_{c_i}(\frac{\Delta(\alpha + 2)}{8(\alpha + 1)})$ \label{line:next_version_space}
			\EndIf
			\EndFor
			\State $ i\gets i+1 $
			\EndLoop
		\end{algorithmic}

	\end{algorithm}

\subsection{Partial Knowledge of the Metric Space }
 \label{sec:partial-knowledge}
Despite the theoretical importance of \Algone, it suffers from two disadvantages: (i) First, it needs to know the pairwise distances between all points a priori, which makes it impractical in many applications. (ii) Second, experimental evaluations show that, due to the high degree of the polynomial function $f(c)$ (see proof of \cref{thm:workcs-I}), the query complexity of \Algone is worse than baseline algorithms for small datasets. 
To overcome these limitations, we present another algorithm called \Algtwo.

\Algtwo allows us to relax the assumption of having access to all the pairwise distances.
Indeed, the \Algtwo algorithm can work by knowing only
relative distances of points, i.e., for each triplets $x,y$ and $z$, it needs to know only if $\alpha d(x,y) \leq d(x,z) $ or $\alpha d(x,z) \leq d(x,y) $. Note that this is enough to find  $Vor(x, y, \mathcal{V}_i) $ and $Vor(y, x, \mathcal{V}_i)$ for all points $x$ and $y$.
We refer to this information as $\mathcal{F}_1$. Information of $\mathcal{F}_1$ is applied to find the resulting version space after each query.
 \cref{thm:average_number_of_calls} guarantees that \Algtwo, by using $\mathcal{F}_1$ and choosing a pair with a $\beta$ approximation of the largest distance in each iteration, will lead to a very competitive algorithm.

	\begin{algorithm}[htb!]
	\caption{\Algtwo \label{alg:main}}
		\begin{algorithmic}[1]
	\Require  Oracle $\mathcal{O}_t(\cdot, \cdot)$, $\mathcal{F}_1$ and $\mathcal{F}_2$ \Comment{$\mathcal{F}_1$ is used for finding the next version space after each query. Auxiliary information $\mathcal{F}_2$ is used for finding points with larger $\beta$.}
			\Ensure Target $t$
			\State $  \mathcal{V}_0 \gets \mathcal{M} $ and $ i\gets 0 $
			\Loop
			\If{$ |\vv_i| \leq 2 $}
			\State Invoke $ \mathcal{O}_t(x,y) $ and halt
			\EndIf
			\State Find $ x,y\in \mathcal{V}_i $ such that $ d(x,y) \geq \beta \diam(\mathcal{V}_i) $ \label{line:find}
			\If{$ \mathcal{O}_t(x,y) = x $}
			\State $ \mathcal{V}_{i+1} =\mathcal{V}_i  \setminus Vor(y, x, \mathcal{V}_i) $\label{alg:line_s1}
			\ElsIf{$ \mathcal{O}_t(x,y) = y $}  \label{alg:line_s2}
			\State $ \mathcal{V}_{i+1} =\mathcal{V}_i \setminus Vor(x, y, \mathcal{V}_i) \} $  \label{alg:line_s2}
			\Else
			\State $ \mathcal{V}_{i+1} = \mathcal{V}_i \setminus (Vor(x, y, \mathcal{V}_i) \cup Vor(y, x, \mathcal{V}_i)) $  \label{alg:line_s3}
			\EndIf
			\State $ i\gets i+1 $
			\EndLoop
		\end{algorithmic}	

\end{algorithm}

\begin{theorem}\label{thm:average_number_of_calls}
	Let $ l \triangleq \left \lceil \log_2 \left( \frac{\alpha+1}{\beta} \right) \right \rceil$.
	The average number of required queries by \Algtwo to locate an object is bounded from above by 
	$1 + \frac{H(\mu)}{\log(1/ (1-\Cstrong^{-l}) )}$, 
	where $\Cstrong$ is the strong doubling constant of the measure $\mu$ in the metric space $\mathcal{M}$.
\end{theorem} 

	\begin{proof}

	Let  $ S_1 \triangleq \mathcal{V}_i \setminus  Vor(y, x, \mathcal{V}_i),  S_2 \triangleq \mathcal{V}_i \setminus  Vor(x, y, \mathcal{V}_i) $ and $ S_3\triangleq \mathcal{V}_i \setminus (Vor(x, y, \mathcal{V}_i) \cup Vor(y, x, \mathcal{V}_i)) $. We denote the distance between $ x $ and $ y $ by $ r  \triangleq d(x,y) $.
	Assume $\Delta$ is the largest distance between any two pints in $\mathcal{V}_i$, i.e., $ \Delta \triangleq \diam(\mathcal{V}_i) $.
	We have $\beta =  \Delta / r$ for  $0 \leq \beta \leq 1$.
	We condition on the target element $ t\in \supp(\m) $. 
	
	We first prove that	$ \mu(\vv_i) \leq (1-\Cstrong^{-l})^i\mu(\vv_0)  =(1-\Cstrong^{-l})^i. $
	Note that we have $2^l \cdot \tfrac{r}{\alpha+1} \geq \tfrac{\alpha+1}{\beta} \cdot \tfrac{r}{\alpha+1} \geq \Delta. $
	The first step is to show that $ Vor(x, y, \mathcal{V}_i) \supseteq B_x(\frac{r}{\alpha+1}) $.
	For any element $ v  \in B_x(\frac{r}{\alpha+1}) $,
	we have $ d(x,v) \leq \frac{r}{\alpha+1} $.
	Therefore, $\alpha d(x,v) \leq \tfrac{ \alpha r}{\alpha+1} \leq r - d(x,v) = d(x,y) - d(x,v) \leq d(y,v), $
	which yields immediately that $ v\in Vor(x, y, \mathcal{V}_i)$.
	As a result,
	\begin{align*}
	  \mu(Vor(x, y, \mathcal{V}_i))  \geq \mu(B_x(\frac{r}{\alpha+1})) & \geq 
 \Cstrong^{-l} \mu( B_x(2^l \cdot \frac{r}{\alpha+1}) )   
	 	\\ &  \geq \Cstrong^{-l} \mu( B_x(D) ) \geq \Cstrong^{-l} \mu(\mathcal{V}_i).
	\end{align*}

	We deduce that 	$\mu(S_2)  = \mu(\mathcal{V}_i) - \mu(Vor(x, y, \mathcal{V}_i)) \leq (1-\Cstrong^{-l}) \mu(\mathcal{V}_i) .$
	Similarly, we have $\mu(Vor(y, x, \mathcal{V}_i)) \geq \Cstrong^{-l} \mu(\mathcal{V}_i) $
	and  $ \mu(S_1) \leq (1-\Cstrong^{-l}) \mu(\mathcal{V}_i) .$ In addition, $\mu(S_3) \leq  (1-\Cstrong^{-l}) \mu(\mathcal{V}_i)$.
	To sum up, we have $ \max_{1\leq j\leq 3}{\mu(S_j)}  \leq (1-\Cstrong^{-l}) \mu(\mathcal{V}_i).$

	The search process ends after at most $ i $ iterations provided $ \mu(\vv_i) \leq (1-\Cstrong^{-l})^i \mathcal{\vv_0} \leq \mu(t),$ or equivalently
	$ i \geq 1 + \tfrac{\log \mu(t) }{ \log(1-\Cstrong^{-l})  } .$
	The average number of iterations is then bounded from above by
\begin{align*}
 \sum_{t \in \supp(\m)} \mu(t) \left( 1 + \frac{\log \mu(t)}{ \log(1-\Cstrong^{-l})}  \right) 
\leq 1 + \frac{H(\mu)}{\log(1/ (1-\Cstrong^{-l}) )}.
\end{align*}

	Also, in each iteration we need to query only one pair of objects.
\end{proof}

From \cref{thm:average_number_of_calls},  we observe that the performance of \Algtwo improves by larger values of $\beta$.
Generally, we can assume there is auxiliary information (called $\mathcal{F}_2)$ that could be applied to find points with larger distances in each iteration.
For example, \AlgtwoRank refers to a version of \Algtwo where the ranking relationships between objects is provided as side information (it is used for $\mathcal{F}_2$). 
By knowing the rankings, we can easily find the farthest points from $x$. From the triangle inequality, we can ensure this results in a $\beta$ of at least $1/2.$ 
Also, \Algtwo by knowing all the pairwise distances (similar to \Algone), can always find a query with $\beta = 1$.
For the detailed explanation of \Algtwo see \cref{alg:main}.

Although the theoretical guarantee for \Algtwo depends on the assumption that the underlying metric space is $\Cstrong$-doubling, the dependence in $\Cstrong$ is through a polynomial function with a degree much smaller than $f(c)$.
Note that the \Algtwo algorithm finds the next query much faster than \AlgGBS. Indeed, \AlgGBS needs $O(|V_i|^3)$ operations in order to find the next query at each iteration. It is easy to see, for example, \AlgtwoRank finds the next query in only $O(|V_i|)$ steps.
In the following, we present a version of \Algtwo that needs much less information in order to find an acceptable query.

\subsection{Algorithm with Minimalistic Information} \label{sec:worcs-weak}
In this section, we present \AlgtwoWeak which needs only partial information about the relative distances. This is exactly the information from $\mathcal{F}_1$.
\AlgtwoWeak, in each iteration, instead of picking two points with distance $ \diam(\mathcal{V}_i)$ or $\diam(\mathcal{V}_i)/2$ (which is possible only with information provided from $\mathcal{F}_2$), uses \cref{alg:pickxy} to find the next query.
From the result of \cref{lemma:maximumdistestimate2}, we guarantee that \cref{alg:pickxy} finds a pair of points with distance at least $\frac{\Delta}{2\alpha}$.
The computational complexity to find such a pair is $O(|\mathcal{V}_i|^2)$.

\begin{lemma} \label{lemma:maximumdistestimate2}
	Assume for $x$ and $y$ there exist no $z \neq y$ such that $\mathcal{O}_x(y,z) = y$, then $d(x,y) \geq \frac{\Delta}{2\alpha}$.
\end{lemma}
	\begin{proof}
	We first prove that we can always find at least one point with this property. 
	Define $x^* = \argmax_{z \in \mathcal{V}_i} d(x,z)$.
	From \cref{lemma:maximumdist}, we know that $ d(x, x^*) \geq \frac{\Delta}{2}$. 
	We claim there is no $z \neq x^* $ such that $x \in Vor(x^*, z, \mathcal{V}_i) $.
	 Assume there is a $z$. This means $\alpha d(x, x^*)  \leq d(x, z)$, where it contradicts with the choice of $x^*$. This means that the set of points with this property is not empty.
	If $y = z^*$ then we are done with the proof, because $d(x, x^*) \geq \frac{\Delta}{2} \geq \frac{\Delta}{2\alpha}$.
	Next, we prove that for any $y \neq x^*$ with this property, we have $d(x,y) \geq \frac{\Delta}{2\alpha}$.
	Assume $y \neq x^*$.  We know $x \notin Vor(y, x^*, \mathcal{V}_i) $. Therefore, we have $\alpha d(x,y) \geq d(x, x^*) \geq \frac{\Delta}{2}$ and $d(x,y) \geq  \frac{\Delta}{2\alpha}.$
\end{proof}

\begin{corollary} \label{cor:workII-weak}
	 From the results of \cref{thm:average_number_of_calls} and \cref{lemma:maximumdistestimate2}, we conclude that the expected query complexity of \AlgtwoWeak is bounded from above by $ 1 + \tfrac{H(\mu)}{\log(1/ (1-\Cstrong^{-l}) )}$ with $l = 1 + \lceil \log_2 \alpha(\alpha + 1) \rceil$.
\end{corollary} 

 The bounds in \cref{thm:average_number_of_calls} and \cref{cor:workII-weak} depend on the value of $\Cstrong$ through the required number of iterations because the reduction in the size of version space is lower bounded by a function of $\Cstrong$. In \Algtwo at each iteration we only ask one question, therefore we do not have any dependence on a polynomial function of $\Cstrong$. 
 Also, note that our algorithms do need to know the value of $\alpha$.
 
\begin{algorithm}[H]
	\caption{Find $x$ and $y$ \label{alg:pickxy}}
	\begin{algorithmic}[1]
		\Require $\mathcal{V}_i $ and $\mathcal{F}_1$ 
		\Ensure $(x,y)$ with $d(x,y) \geq \frac{1}{2\alpha} \diam(\mathcal{V}_i)$
		\State Uniformly at random pick $x \in \mathcal{V}_i$
		\For{$y \in \mathcal{V}_i \setminus \{ x \}$}
			\If{$\forall z \in \mathcal{V}_i \setminus \{ x , y\}: x \notin Vor(y, z, \mathcal{V}_i) $}
		\State \Return{$(x,y)$}
			\EndIf
		\EndFor
	\end{algorithmic}

\end{algorithm}

In this section, we presented four different algorithms to locate an object in a database by using a weak oracle model. These algorithms use different types of information as input. To provide guarantees for their performances, we need to make different assumptions on the structure of the underlying metric space (see \cref{table:alg-compare}).  

\Algone provides better theoretical results; The guarantee for \Algone depends on constant values of $c$ which is a looser assumption than having a constant $\Cstrong.$ 
In the next section, we will show that although the theoretical guarantee of \Algtwo is based on a stronger assumption over the metric space ($\Cstrong$ vs $c$), in all of our experiments it shows better performances in comparison to \Algone. 
It finds the target with fewer queries and the computational complexity of choosing the next query is much less. We believe that for most real datasets both $c$ and $\Cstrong$ are small and close to each other. Therefore, the polynomial term $f(c)$ (defined in \cref{thm:workcs-I}) plays a very important role in practice.
Finally, providing guarantees for the performance of \AlgGBS or similar algorithms seems impossible \citep{dasgupta2004analysis,golovin2011adaptive}.

\begin{table*}[t]
	\caption{Comparison between algorithms}
	\label{table:alg-compare}
	\centering
	\begin{tabular}{lll}
	\toprule[1.5pt]
		Algorithm & Input information & Constraint          \\
		\midrule
		\AlgGBS & Weak rank information  &  No guarantees \\
		\Algone & Pairwise distances & Doubling constant    \\
		\AlgtwoRank & Rank information & Strong doubling constant    \\
		\AlgtwoWeak & Weak rank information & Strong doubling constant  \\
		 \bottomrule[1.25pt]
	\end{tabular}
\end{table*}

\section{Experiments} \label{sec:experiments}
We compare \Algone and \Algtwo (\AlgtwoWeak and \AlgtwoRank)\footnote{In this section, \textsc{Worcs-II-R} and \textsc{Worcs-II-W} stand for \AlgtwoRank and  \AlgtwoWeak.} with several baseline methods. 
Recall that \Algone needs to know the pairwise distances between objects, and \AlgtwoRank and \AlgtwoWeak need only information about ranking and partial ordering obtained  from the weak oracle, respectively.
For choosing the baselines, we followed the same approach as \citet{karbasi2012comparison}.
Our baseline methods are:
\begin{itemize}
	\item \AlgGBS which is explained in \cref{sec:gts}.
 	\item \AlgRand: The general framework of \AlgRand is the same as \cref{alg:main} except that in Line \ref{line:find}, \AlgRand randomly samples a pair of points $ x $ and $ y $ in the current version space.
 \item \AlgFastGBS: In light of the  computationally prohibitive nature of \AlgGBS, \AlgFastGBS is an approximate alternative to \AlgGBS. Rather than finding the exact minimizer, it randomly samples $ k $ pairs of points and finds the minimizer of \cref{eq:gts} only in the sampled pairs. 
	\AlgRand and \AlgGBS can be viewed as a special case of \AlgFastGBS with parameter $ k=1 $ and $ k = \binom{|\vv_i|}{2} $ (sampling without repetition), respectively.
	We take $ k=10 $ in the experiments.
	\item  \AlgMinDist: In contrast to \Algtwo, \AlgMinDist selects a pair of points that reside closest to each other. We use this baseline to highlight the importance of choosing a pair that attain the approximate maximum distance.
\end{itemize}


We evaluate the performance of algorithms through two metrics: (i) query complexity: the expected number of queries to locate an object, and (ii) computational complexity: the total time-complexity of determining which
queries to submit to the oracle.
We conduct experiments over the following datasets: \emph{MovieLens}~\citep{harper2016movielens}, \emph{Wine}~\citep{forina1991extendible}, \emph{Music}~\citep{zhou2014predicting}, and Fisher's \emph{Iris} dataset.
The demand distribution $ \mu $ is set to the power law with exponent $ 0.4 $.
In order to model uncertain  answers from the oracle  $ \mathcal{O}_t(x,y) $, i.e., for points $t$ that are almost equidistant from both $x$ and $y$, we use the following model:
if $ d(x,t) \leq d(y,t) < \alpha d(x,t)  $, the oracle outputs $ x $ with probability $ \tfrac{\log(d(y,t)/d(x,t))}{\log \alpha} $ and outputs $ ``?" $ otherwise; similarly if $ d(y,t) \leq d(x,t) < \alpha d(y,t) $, the oracle outputs $ y $ with probability $\tfrac{\log(d(x,t)/d(y,t)}{\log \alpha}$ and outputs $ ``?" $ otherwise.

We present expected query complexity of all seven algorithms on each dataset in~\cref{tab:query} (we have to use a table rather than a bar chart since the data are highly skewed). 
The corresponding computational complexity is shown in~\cref{fig:computational_complexity}. It can be observed that \AlgtwoWeak outperforms all other algorithms in terms of query complexity.
 \AlgtwoRank is only second to \AlgtwoWeak. Intuitively, selecting two distant items for a query leads into two large Voronoi cells around them, which partially explains the good performance of \AlgtwoRank.
 However, it could leave the rest of the version space relatively small; therefore, it may not reduce the version space substantially when a $``?"$ response is received from the oracle. Unlike \AlgtwoRank,
   \AlgtwoWeak is only guaranteed to find a pair that attain a $ \beta $-approximation of the diameter of the current version space, where $ \beta \geq \frac{1}{2\alpha} $ as shown in~\cref{lemma:maximumdistestimate2}, thereby resulting in a relatively more balanced division of the version space. This explains the slightly better performance of \AlgtwoWeak versus \AlgtwoRank.
 With respect to the computational complexity, \AlgRand costs the smallest amount of computational resources since it does not need to algorithmically decide which pair of items to query. \AlgtwoWeak is only second to \AlgRand.

\begin{table}[t]
	\centering
	\caption{Expected Query Complexity on Each Dataset} \label{tab:query}
	\begin{tabular}{llllllll}
\toprule[1.5pt]
		& \textsc{Rand}         & \Algone & \textsc{W-II-R}\footnote{\textsc{W-II-R} and \textsc{W-II-W} are the abbreviation of \AlgtwoRank and \AlgtwoWeak, respectively.} & \textsc{W-II-W}  & \AlgMinDist & \AlgGBS    & \AlgFastGBS \\ \midrule
		 \textit{MovieLens} & 11.46 & 594.97 & 10.86 & 10.79 & 49.45 & \textbf{9.01}  & 9.79 \\
		\textit{Wine} & 7.03  & 55.68 & 6.47  & \textbf{6.11} & 35.16 & 22.82 & 10.46 \\
		\textit{Music} & 14.31 & 648.12 & 12.71 & \textbf{12.67} & 143.61 & 36.90 & 27.78 \\
		\textit{Iris} & 7.86  & 164.42 & 7.56  & \textbf{6.98} & 19.85 & 9.81  & 9.54 \\
		 \bottomrule[1.25pt]
	\end{tabular}
	\end{table}

\begin{figure*}[h!]
	\centering
	\subfigure[\label{fig:computational_complexity}]{
		\includegraphics[width=0.31\linewidth]{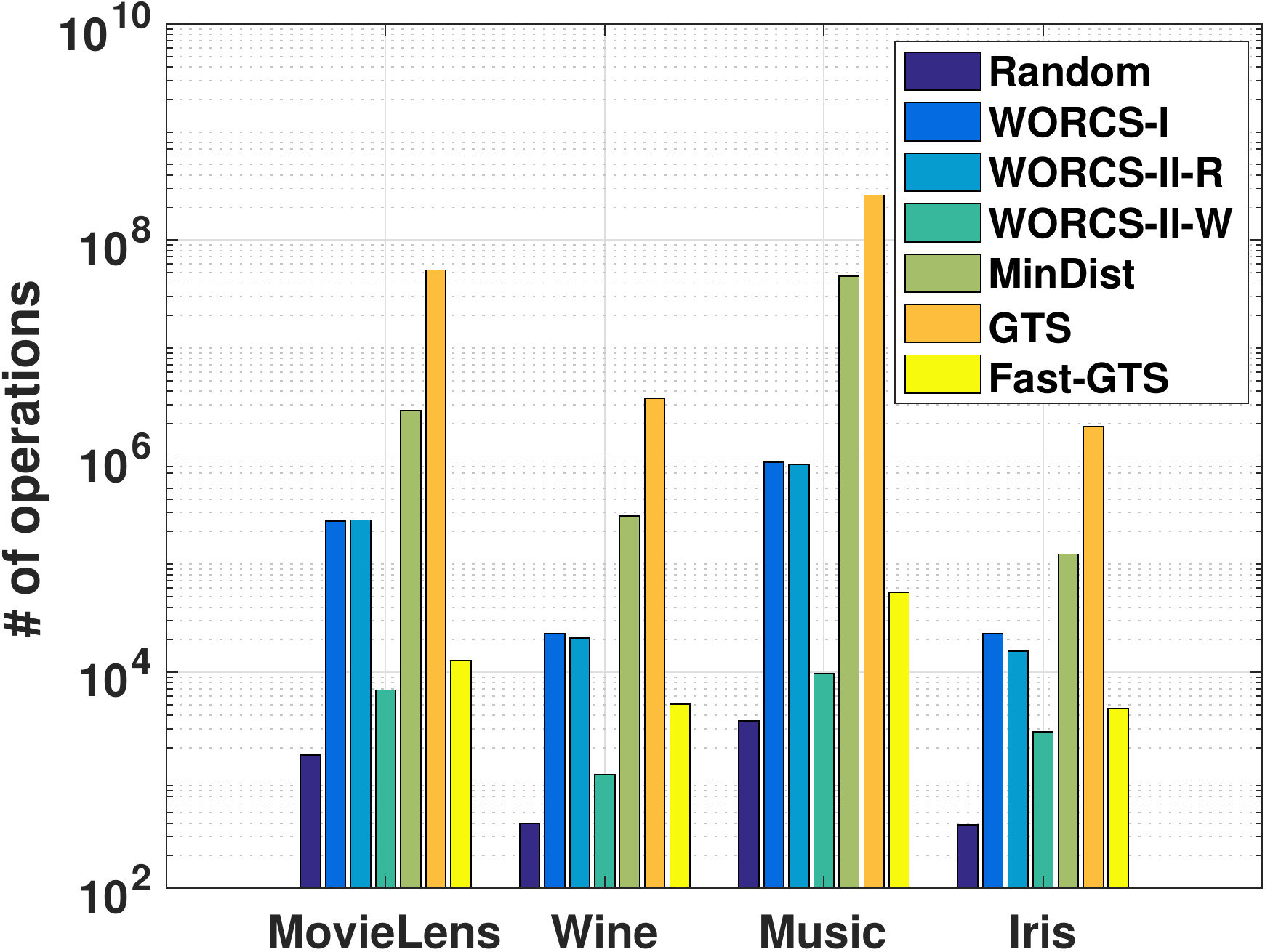}
	}
\subfigure[\label{fig:query_vs_N}]{
	\includegraphics[width=0.31\linewidth]{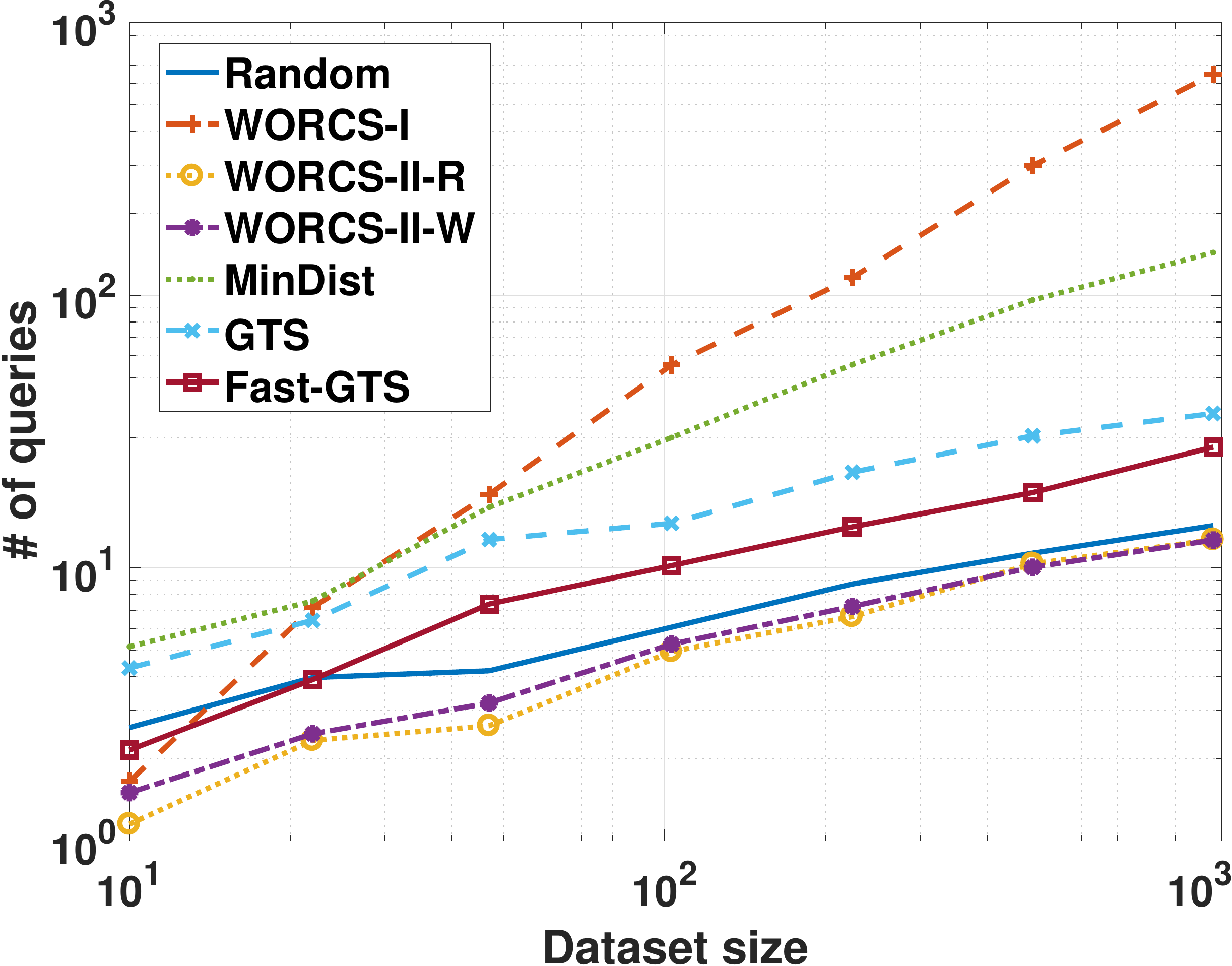}
}
\subfigure[\label{fig:comp_vs_N}]{
		\includegraphics[width=0.31\linewidth]{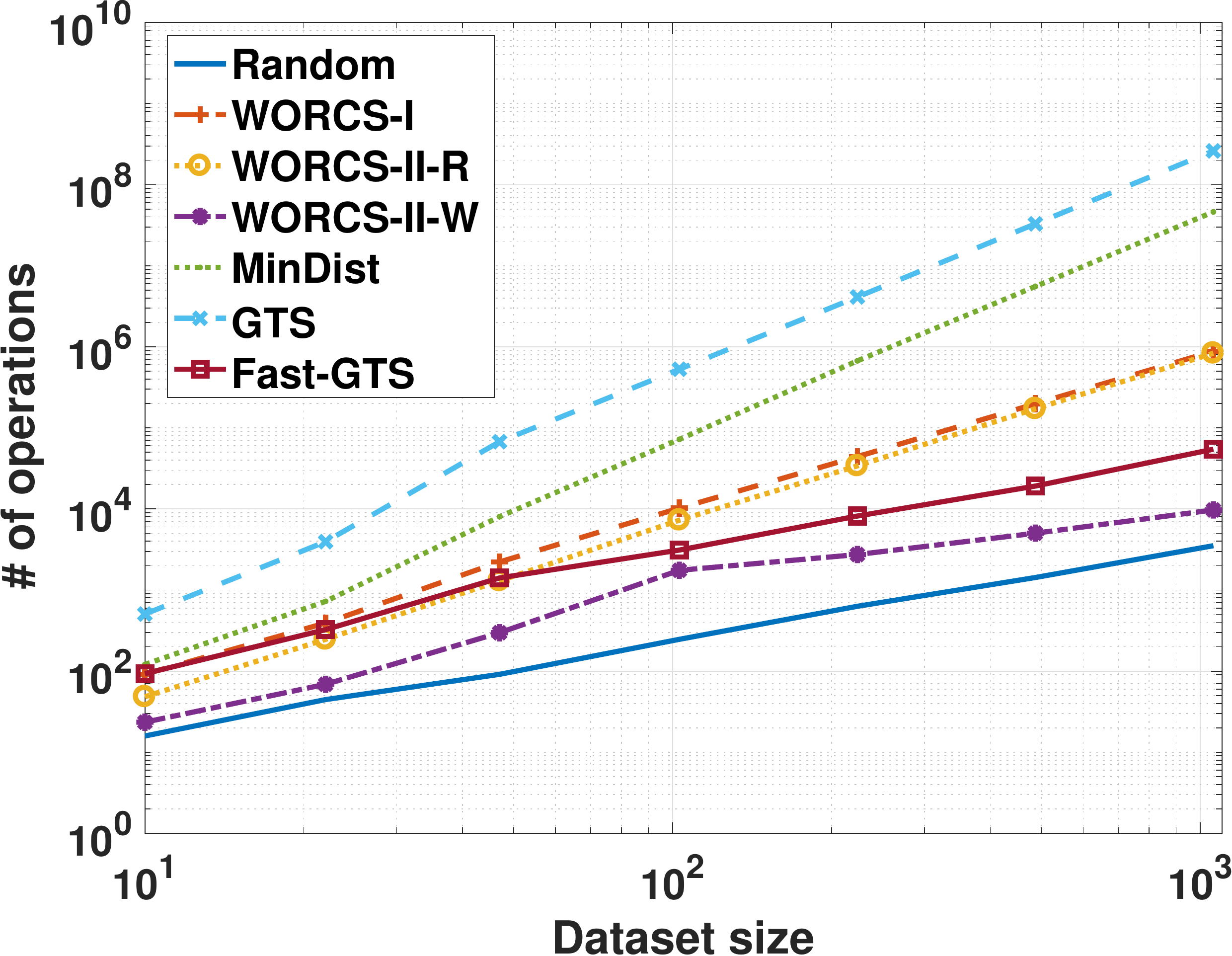}
}
\subfigure[\label{fig:query_vs_a}]{
	\includegraphics[width=0.31\linewidth]{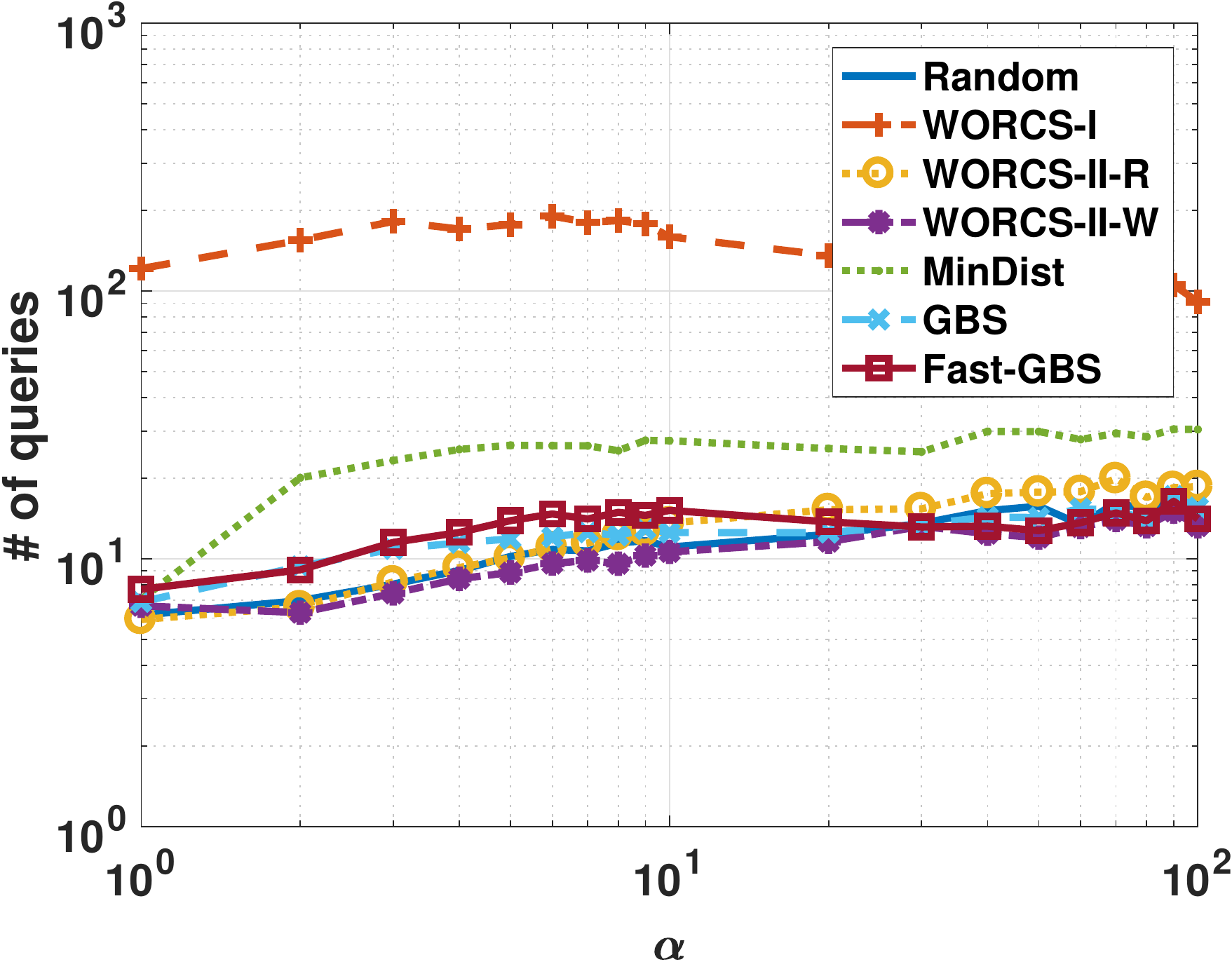}
}\subfigure[\label{fig:comp_vs_a}]{
\includegraphics[width=0.31\linewidth]{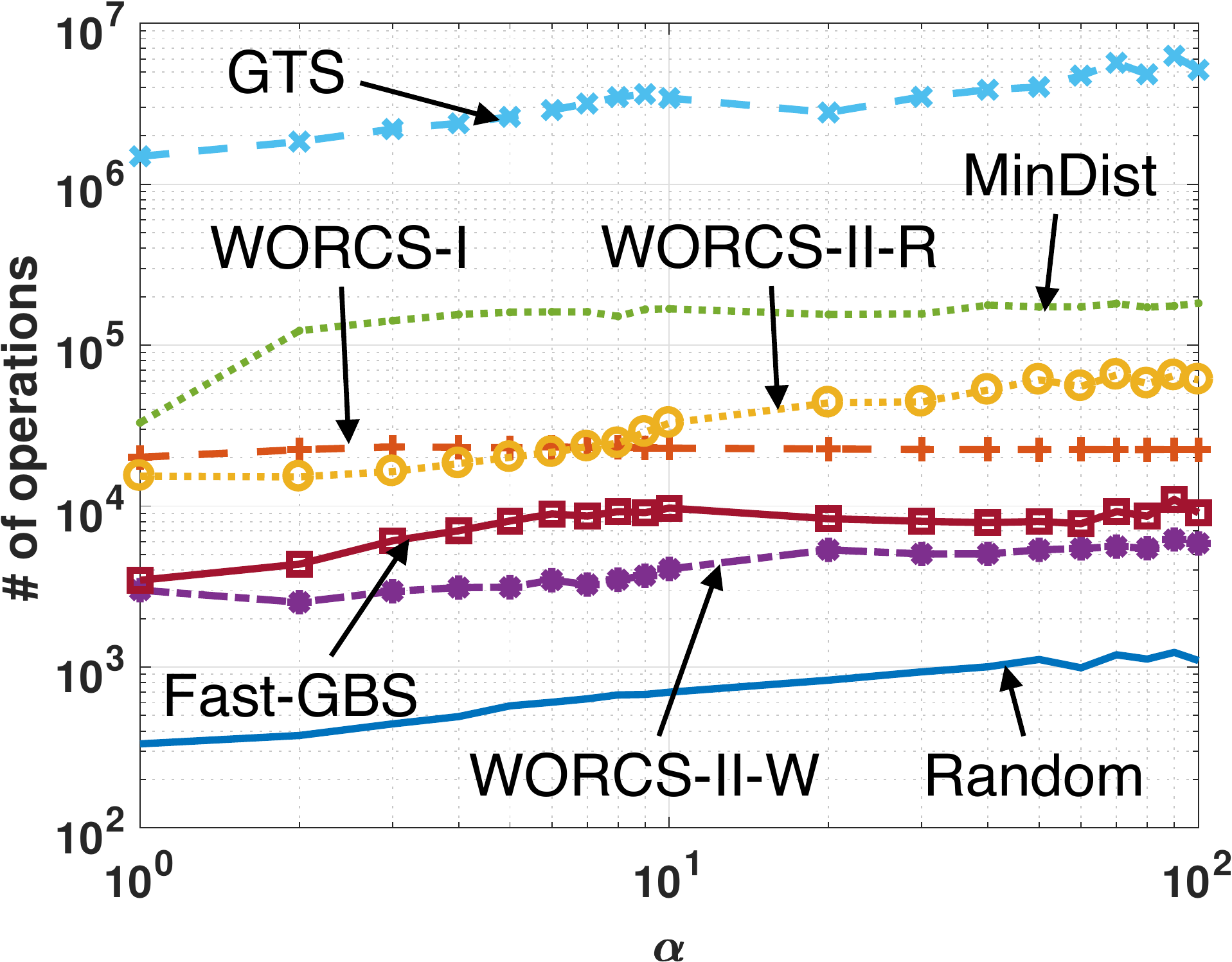}
}\subfigure[\label{fig:query_vs_expnt}]{
\includegraphics[width=0.31\linewidth]{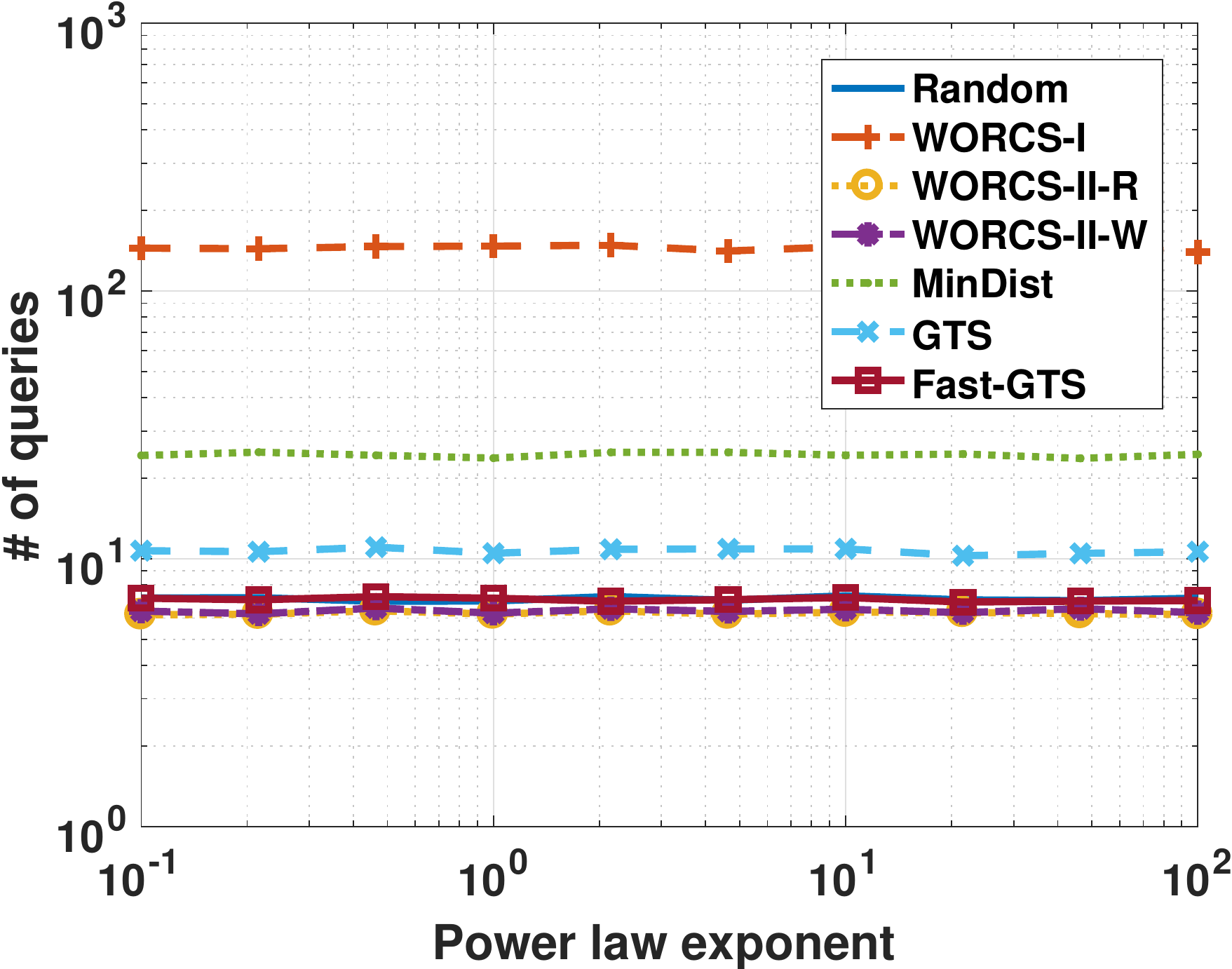}
}
\caption{\cref{fig:computational_complexity} shows the expected computational complexity, per search, of all seven algorithms applied to each dataset. \cref{fig:query_vs_a,fig:comp_vs_a} show the expected query and computational complexity as a function of the dataset size. We present how the expected query and computational complexity  vary with $ \alpha $ in \cref{fig:query_vs_a,fig:comp_vs_a}. In~\cref{fig:query_vs_expnt}, we present how the query complexity  scales as a function of the power law exponent of the demand distribution. }
\end{figure*}

\paragraph{Scalability}
In this set of experiments, in order to study the scalability of algorithms, we use the \emph{Music} dataset and sub-sampled
$ 10 $ to $ 1000 $ items from the dataset.  Query and computational complexity of different algorithms is shown in~\cref{fig:query_vs_N,fig:comp_vs_N}. 
We see the positive correlation between query/computational complexity and the dataset size. Since the curves are approximately linear in a log-log plot, the complexity of algorithms scales approximately polynomially in the dataset size. Specifically, the query complexity of \AlgtwoRank and \AlgtwoWeak scales approximately $ O(\sqrt{N}) $, where $ N $ is the dataset size.
 In~\cref{fig:query_vs_N}, the curves of \AlgtwoRank and \AlgtwoWeak are the two lowest curves, which represent lowest query complexity. 

\paragraph{Impact of $\boldsymbol{\alpha}$}
The performance of algorithms for different values of $\alpha$  is another important aspect to study.
We used the \emph{Iris} dataset and varied $ \alpha $ from $ 1 $ to $ 100 $. Query and computational complexity of different algorithms is shown in~\cref{fig:query_vs_a,fig:comp_vs_a}. A general upward tendency in computational complexity as $ \alpha $ increases is observed in~\cref{fig:comp_vs_a}. With respect to the query complexity shown in~\cref{fig:query_vs_a}, \Algone exhibits a unimodal curve. 
In fact, while a large $ \alpha $ leads to a greater number of balls in the $ \frac{\Delta}{8(\alpha+1)} $-cover (increases the query complexity in line~\ref{line:cover} of~\cref{alg:full_knowledge}), it results in smaller balls and therefore a smaller version space in the next iteration (reduces the query complexity as in line~\ref{line:next_version_space} of~\cref{alg:full_knowledge}). This explains the unimodal behavior of the curve of \Algone. Other algorithms exhibit a general positive correlation with $ \alpha $, as a larger $ \alpha $ tends to induce a more unbalanced division of the version space.  
\paragraph{Impact of Demand Distribution}
Finally, we explore impact of demand distribution.
We used the \emph{Iris} dataset and varied the exponent of the power law distribution from $ 10^{-1} $ to $ 10^2 $. Query complexity of different algorithms is shown in~\cref{fig:query_vs_expnt}. We observe that algorithms are generally robust to the change of the parameter of the power law demand distribution.

\section{Conclusion} \label{sec:conclusion}
\vspace{-10pt}
In this paper, we studied the problem of interactive learning though pairwise comparisons.
We introduced a new weak oracle model to handle noisy situations, under which the oracle answers only when it is sure of the answer. 
We also considered the problem of comparison-based search via our weak oracle.
We proposed two algorithms where they require different levels of knowledge about the distances or rankings between objects.
We guaranteed the performance of these algorithms based on a measure of intrinsic dimensionality for the distribution of points in the underling metric space. Finally, we compared our algorithms with several baseline algorithms on real datasets.
Note that we assumed the oracle is non-malicious and when it is not confident, it answers with ``?''. Although this model is robust to high levels of uncertainty, but considering the effect of erroneous answers from the weak oracle is interesting for future work.

\paragraph{Acknowledgement} Amin Karbasi is supported by DARPA Young Faculty Award (D16AP00046) and AFOSR Young Investigator Award (FA9550-18-1-0160).
Ehsan Kazemi is supported by Swiss National Science Foundation (Early Postdoc.Mobility) under grant
number 168574. 

\bibliographystyle{plainnat}
\bibliography{./tex/reference-list}

\begin{thebibliography}{49}
\providecommand{\natexlab}[1]{#1}
\providecommand{\url}[1]{\texttt{#1}}
\expandafter\ifx\csname urlstyle\endcsname\relax
  \providecommand{\doi}[1]{doi: #1}\else
  \providecommand{\doi}{doi: \begingroup \urlstyle{rm}\Url}\fi

\bibitem[Balcan et~al.(2009)Balcan, Beygelzimer, and
  Langford]{balcan2009agnostic}
Maria{-}Florina Balcan, Alina Beygelzimer, and John Langford.
\newblock Agnostic active learning.
\newblock \emph{J. Comput. Syst. Sci.}, 75\penalty0 (1):\penalty0 78--89, 2009.

\bibitem[Balcan et~al.(2016)Balcan, Vitercik, and White]{balcan2016learning}
Maria{-}Florina Balcan, Ellen Vitercik, and Colin White.
\newblock Learning combinatorial functions from pairwise comparisons.
\newblock In \emph{Proceedings of the 29th Conference on Learning Theory,
  {COLT} 2016, New York, USA, June 23-26, 2016}, pages 310--335, 2016.

\bibitem[Chen et~al.(2013)Chen, Bennett, Collins-Thompson, and
  Horvitz]{chen2013pairwise}
Xi~Chen, Paul~N Bennett, Kevyn Collins-Thompson, and Eric Horvitz.
\newblock Pairwise ranking aggregation in a crowdsourced setting.
\newblock In \emph{Proceedings of the sixth ACM international conference on Web
  search and data mining}, pages 193--202. ACM, 2013.

\bibitem[Clarkson(2006)]{clarkson2006nearest}
Kenneth~L Clarkson.
\newblock Nearest-neighbor searching and metric space dimensions.
\newblock \emph{Nearest-neighbor methods for learning and vision: theory and
  practice}, pages 15--59, 2006.

\bibitem[Cover and Thomas(2012)]{cover2012elements}
Thomas~M Cover and Joy~A Thomas.
\newblock \emph{Elements of information theory}.
\newblock John Wiley \& Sons, 2012.

\bibitem[Dalvi et~al.(2013)Dalvi, Dasgupta, Kumar, and
  Rastogi]{dalvi2013aggregating}
Nilesh Dalvi, Anirban Dasgupta, Ravi Kumar, and Vibhor Rastogi.
\newblock Aggregating crowdsourced binary ratings.
\newblock In \emph{Proceedings of the 22nd international conference on World
  Wide Web}, pages 285--294. ACM, 2013.

\bibitem[Dasgupta(2004)]{dasgupta2004analysis}
Sanjoy Dasgupta.
\newblock Analysis of a greedy active learning strategy.
\newblock In \emph{NIPS}, volume~17, pages 337--344, 2004.

\bibitem[Dasgupta(2005)]{dasgupta2005coarse}
Sanjoy Dasgupta.
\newblock Coarse sample complexity bounds for active learning.
\newblock In \emph{Advances in Neural Information Processing Systems, December
  5-8, 2005, Vancouver, British Columbia, Canada]}, pages 235--242, 2005.

\bibitem[Eriksson(2013)]{eriksson2013learning}
Brian Eriksson.
\newblock Learning to top-k search using pairwise comparisons.
\newblock In \emph{Proceedings of the Sixteenth International Conference on
  Artificial Intelligence and Statistics, {AISTATS} 2013, Scottsdale, AZ, USA,
  April 29 - May 1, 2013}, pages 265--273, 2013.

\bibitem[Firmani et~al.(2016)Firmani, Saha, and Srivastava]{firmani2016online}
Donatella Firmani, Barna Saha, and Divesh Srivastava.
\newblock Online entity resolution using an oracle.
\newblock \emph{Proceedings of the VLDB Endowment}, 9\penalty0 (5):\penalty0
  384--395, 2016.

\bibitem[Forina et~al.(1991)]{forina1991extendible}
M~Forina et~al.
\newblock An extendible package for data exploration, classification and
  correlation.
\newblock \emph{Institute of Pharmaceutical and Food Analisys and Technologies,
  Via Brigata Salerno}, 16147, 1991.

\bibitem[Golovin and Krause(2011)]{golovin2011adaptive}
Daniel Golovin and Andreas Krause.
\newblock Adaptive submodularity: Theory and applications in active learning
  and stochastic optimization.
\newblock \emph{J. Artif. Intell. Res. {(JAIR)}}, 42:\penalty0 427--486, 2011.

\bibitem[Goyal et~al.(2008)Goyal, Lifshits, and Sch\"{u}tze]{goyal2008disorder}
Navin Goyal, Yury Lifshits, and Hinrich Sch\"{u}tze.
\newblock Disorder inequality: A combinatorial approach to nearest neighbor
  search.
\newblock In \emph{Proceedings of the 2008 International Conference on Web
  Search and Data Mining}, WSDM '08, pages 25--32, New York, NY, USA, 2008.
  ACM.

\bibitem[Gupta et~al.(2003)Gupta, Krauthgamer, and Lee]{gupta2003bounded}
Anupam Gupta, Robert Krauthgamer, and James~R Lee.
\newblock Bounded geometries, fractals, and low-distortion embeddings.
\newblock In \emph{Foundations of Computer Science, 2003. Proceedings. 44th
  Annual IEEE Symposium on}, pages 534--543. IEEE, 2003.

\bibitem[Haghiri et~al.(2017)Haghiri, Ghoshdastidar, and von
  Luxburg]{haghiri2017comparison}
Siavash Haghiri, Debarghya Ghoshdastidar, and Ulrike von Luxburg.
\newblock Comparison-based nearest neighbor search.
\newblock In \emph{Artificial Intelligence and Statistics}, 2017.

\bibitem[Har-Peled and Kumar(2013)]{har2013approximate}
Sariel Har-Peled and Nirman Kumar.
\newblock Approximate nearest neighbor search for low-dimensional queries.
\newblock \emph{SIAM Journal on Computing}, 42\penalty0 (1):\penalty0 138--159,
  2013.

\bibitem[Harper and Konstan(2016)]{harper2016movielens}
F~Maxwell Harper and Joseph~A Konstan.
\newblock The movielens datasets: History and context.
\newblock \emph{ACM Transactions on Interactive Intelligent Systems (TiiS)},
  5\penalty0 (4):\penalty0 19, 2016.

\bibitem[Hastie and Tibshirani(1997)]{hastie1997classification}
Trevor Hastie and Robert Tibshirani.
\newblock Classification by pairwise coupling.
\newblock In \emph{Advances in Neural Information Processing Systems 10,
  {[NIPS} Conference, Denver, Colorado, USA, 1997]}, pages 507--513, 1997.

\bibitem[Heckel et~al.(2016)Heckel, Shah, Ramchandran, and
  Wainwright]{heckel2016active}
Reinhard Heckel, Nihar~B Shah, Kannan Ramchandran, and Martin~J Wainwright.
\newblock Active ranking from pairwise comparisons and the futility of
  parametric assumptions.
\newblock \emph{arXiv preprint arXiv:1606.08842}, 2016.

\bibitem[Hildrum et~al.(2004)Hildrum, Kubiatowicz, Rao, and
  Zhao]{hildrum2004distributed}
Kirsten Hildrum, John~D Kubiatowicz, Satish Rao, and Ben~Y Zhao.
\newblock Distributed object location in a dynamic network.
\newblock \emph{Theory of Computing Systems}, 37\penalty0 (3):\penalty0
  405--440, 2004.

\bibitem[Huang et~al.(2010)Huang, Jin, and Zhou]{huang2010active}
Sheng-Jun Huang, Rong Jin, and Zhi-Hua Zhou.
\newblock Active learning by querying informative and representative examples.
\newblock In \emph{Advances in neural information processing systems}, pages
  892--900, 2010.

\bibitem[Indyk and Motwani(1998)]{indyk1998approximate}
Piotr Indyk and Rajeev Motwani.
\newblock Approximate nearest neighbors: towards removing the curse of
  dimensionality.
\newblock In \emph{Proceedings of the thirtieth annual ACM symposium on Theory
  of computing}, pages 604--613. ACM, 1998.

\bibitem[Ipeirotis et~al.(2010)Ipeirotis, Provost, and
  Wang]{ipeirotis2010quality}
Panagiotis~G Ipeirotis, Foster Provost, and Jing Wang.
\newblock Quality management on amazon mechanical turk.
\newblock In \emph{Proceedings of the ACM SIGKDD workshop on human
  computation}, pages 64--67. ACM, 2010.

\bibitem[Karbasi et~al.(2011)Karbasi, Ioannidis, and
  Massouli{\'{e}}]{karbasi2011content}
Amin Karbasi, Stratis Ioannidis, and Laurent Massouli{\'{e}}.
\newblock Content search through comparisons.
\newblock In \emph{Automata, Languages and Programming - 38th International
  Colloquium, {ICALP} 2011, Zurich, Switzerland, July 4-8, 2011, Proceedings,
  Part {II}}, pages 601--612, 2011.

\bibitem[Karbasi et~al.(2012{\natexlab{a}})Karbasi, Ioannidis, and
  Massouli{\'{e}}]{karbasi2012comparison}
Amin Karbasi, Stratis Ioannidis, and Laurent Massouli{\'{e}}.
\newblock {Comparison-Based Learning with Rank Nets}.
\newblock In \emph{Proceedings of the 29th International Conference on Machine
  Learning, {ICML} 2012, Edinburgh, Scotland, UK, June 26 - July 1, 2012},
  2012{\natexlab{a}}.

\bibitem[Karbasi et~al.(2012{\natexlab{b}})Karbasi, Ioannidis, and
  Massouli{\'e}]{karbasi2012hot}
Amin Karbasi, Stratis Ioannidis, and Laurent Massouli{\'e}.
\newblock Hot or not: Interactive content search using comparisons.
\newblock In \emph{Information Theory and Applications Workshop (ITA), 2012},
  pages 291--297. IEEE, 2012{\natexlab{b}}.

\bibitem[Karger and Ruhl(2002)]{karger2002finding}
David~R Karger and Matthias Ruhl.
\newblock Finding nearest neighbors in growth-restricted metrics.
\newblock In \emph{Proceedings of the thiry-fourth annual ACM symposium on
  Theory of computing}, pages 741--750. ACM, 2002.

\bibitem[Karp and Kleinberg(2007)]{karp2007noisy}
Richard~M Karp and Robert Kleinberg.
\newblock Noisy binary search and its applications.
\newblock In \emph{Proceedings of the eighteenth annual ACM-SIAM symposium on
  Discrete algorithms}, pages 881--890. Society for Industrial and Applied
  Mathematics, 2007.

\bibitem[Krauthgamer and Lee(2004)]{krauthgamer2004navigating}
Robert Krauthgamer and James~R Lee.
\newblock {Navigating nets: simple algorithms for proximity search}.
\newblock In \emph{Proceedings of the fifteenth annual ACM-SIAM symposium on
  Discrete algorithms}, pages 798--807. Society for Industrial and Applied
  Mathematics, 2004.

\bibitem[Lifshits and Zhang(2009)]{lifshits2009combinatorial}
Yury Lifshits and Shengyu Zhang.
\newblock Combinatorial algorithms for nearest neighbors, near-duplicates and
  small-world design.
\newblock In \emph{Proceedings of the twentieth Annual ACM-SIAM Symposium on
  Discrete Algorithms}, pages 318--326. Society for Industrial and Applied
  Mathematics, 2009.

\bibitem[Maystre and Grossglauser(2015)]{maystre2015fast}
Lucas Maystre and Matthias Grossglauser.
\newblock Fast and accurate inference of plackett--luce models.
\newblock In \emph{Advances in Neural Information Processing Systems}, pages
  172--180, 2015.

\bibitem[Maystre and Grossglauser(2017)]{maystre2017just}
Lucas Maystre and Matthias Grossglauser.
\newblock {Just Sort It! {A} Simple and Effective Approach to Active Preference
  Learning}.
\newblock In \emph{Proceedings of the 34th International Conference on Machine
  Learning, {ICML} 2017, Sydney, NSW, Australia, 6-11 August 2017}, pages
  2344--2353, 2017.

\bibitem[Nowak(2009)]{nowak2009noisy}
Robert~D. Nowak.
\newblock Noisy generalized binary search.
\newblock In \emph{Advances in Neural Information Processing Systems}, pages
  1366--1374, 2009.

\bibitem[Qian et~al.(2015)Qian, Gao, and Jagadish]{qian2015learning}
Li~Qian, Jinyang Gao, and HV~Jagadish.
\newblock Learning user preferences by adaptive pairwise comparison.
\newblock \emph{Proceedings of the VLDB Endowment}, 8\penalty0 (11):\penalty0
  1322--1333, 2015.

\bibitem[Ram and Gray(2013)]{ram2013which}
P.~Ram and A.~G. Gray.
\newblock Which space partitioning tree to use for search?
\newblock In \emph{Proceedings of the 26th International Conference on Neural
  Information Processing Systems}, NIPS'13, pages 656--664, USA, 2013.

\bibitem[Salganik and Levy(2015)]{salganik2015wiki}
Matthew~J Salganik and Karen~EC Levy.
\newblock Wiki surveys: Open and quantifiable social data collection.
\newblock \emph{PloS one}, 10\penalty0 (5):\penalty0 e0123483, 2015.

\bibitem[Settles(2010)]{settles2010active}
Burr Settles.
\newblock Active learning literature survey.
\newblock \emph{University of Wisconsin, Madison}, 52\penalty0
  (55-66):\penalty0 11, 2010.

\bibitem[Stewart et~al.(2005)Stewart, Brown, and Chater]{stewart2005absolute}
Neil Stewart, Gordon~DA Brown, and Nick Chater.
\newblock Absolute identification by relative judgment.
\newblock \emph{Psychological review}, 112\penalty0 (4):\penalty0 881, 2005.

\bibitem[Tamuz et~al.(2011)Tamuz, Liu, Belongie, Shamir, and
  Kalai]{tamuz2015adaptively}
Omer Tamuz, Ce~Liu, Serge~J. Belongie, Ohad Shamir, and Adam Kalai.
\newblock Adaptively learning the crowd kernel.
\newblock In \emph{Proceedings of the 28th International Conference on Machine
  Learning, {ICML} 2011, Bellevue, Washington, USA, June 28 - July 2, 2011},
  pages 673--680, 2011.

\bibitem[Thurstone(1927)]{thurstone1927method}
Louis~L Thurstone.
\newblock The method of paired comparisons for social values.
\newblock \emph{The Journal of Abnormal and Social Psychology}, 21\penalty0
  (4):\penalty0 384, 1927.

\bibitem[Tschopp et~al.(2011)Tschopp, Diggavi, Delgosha, and
  Mohajer]{tschopp2011randomized}
Dominique Tschopp, Suhas Diggavi, Payam Delgosha, and Soheil Mohajer.
\newblock Randomized algorithms for comparison-based search.
\newblock In \emph{Advances in Neural Information Processing Systems}, pages
  2231--2239, 2011.

\bibitem[Tschopp et~al.(2015)Tschopp, Diggavi, and
  Grossglauser]{tschopp2015hierarchical}
Dominique Tschopp, Suhas Diggavi, and Matthias Grossglauser.
\newblock Hierarchical routing over dynamic wireless networks.
\newblock \emph{Random Structures \& Algorithms}, 47\penalty0 (4):\penalty0
  669--709, 2015.

\bibitem[Wah et~al.(2014)Wah, Van~Horn, Branson, Maji, Perona, and
  Belongie]{wah2014similarity}
Catherine Wah, Grant Van~Horn, Steve Branson, Subhransu Maji, Pietro Perona,
  and Serge Belongie.
\newblock Similarity comparisons for interactive fine-grained categorization.
\newblock In \emph{Proceedings of the IEEE Conference on Computer Vision and
  Pattern Recognition}, pages 859--866, 2014.

\bibitem[Wang et~al.(2014)Wang, Srebro, and Evans]{wang2014active}
Jialei Wang, Nathan Srebro, and James Evans.
\newblock Active collaborative permutation learning.
\newblock In \emph{Proceedings of the 20th ACM SIGKDD international conference
  on Knowledge discovery and data mining}, pages 502--511. ACM, 2014.

\bibitem[Wang et~al.(2012)Wang, Kraska, Franklin, and Feng]{wang2012crowder}
Jiannan Wang, Tim Kraska, Michael~J Franklin, and Jianhua Feng.
\newblock Crowder: Crowdsourcing entity resolution.
\newblock \emph{Proceedings of the VLDB Endowment}, 5\penalty0 (11):\penalty0
  1483--1494, 2012.

\bibitem[Wang et~al.(2013)Wang, Li, Kraska, Franklin, and
  Feng]{wang2013leveraging}
Jiannan Wang, Guoliang Li, Tim Kraska, Michael~J Franklin, and Jianhua Feng.
\newblock Leveraging transitive relations for crowdsourced joins.
\newblock In \emph{Proceedings of the 2013 ACM SIGMOD International Conference
  on Management of Data}, pages 229--240. ACM, 2013.

\bibitem[Wauthier et~al.(2013)Wauthier, Jordan, and
  Jojic]{wauthier2013efficient}
Fabian~L Wauthier, Michael~I Jordan, and Nebojsa Jojic.
\newblock Efficient ranking from pairwise comparisons.
\newblock \emph{ICML (3)}, 28:\penalty0 109--117, 2013.

\bibitem[Yan et~al.(2011)Yan, Fung, Rosales, and Dy]{yan2011active}
Yan Yan, Glenn~M Fung, R{\'o}mer Rosales, and Jennifer~G Dy.
\newblock Active learning from crowds.
\newblock In \emph{Proceedings of the 28th international conference on machine
  learning (ICML-11)}, pages 1161--1168, 2011.

\bibitem[Zhou et~al.(2014)Zhou, Claire, and King]{zhou2014predicting}
Fang Zhou, Q~Claire, and Ross~D King.
\newblock Predicting the geographical origin of music.
\newblock In \emph{Data Mining (ICDM), 2014 IEEE International Conference on},
  pages 1115--1120. IEEE, 2014.

\end{thebibliography}

\end{document}